\documentclass[10pt,twocolumn,letterpaper]{article}
\usepackage[letterpaper,left=0.75in,right=0.75in,top=0.75in,bottom=1in]{geometry}
\usepackage{mathptmx}
\usepackage{titlesec}
\titleformat{\section}{\centering\normalsize\scshape}{\Roman{section}.}{0.6em}{}
\titleformat{\subsection}{\normalsize\itshape}{\Alph{subsection}.}{0.6em}{}
\titleformat{\subsubsection}{\normalsize\itshape}{\arabic{subsubsection})}{0.6em}{}

\titlespacing*{\section}{0pt}{1.4ex plus .4ex}{0.8ex}
\titlespacing*{\subsection}{0pt}{1.2ex plus .3ex}{0.6ex}
\usepackage{caption}
\usepackage{cite}
\usepackage{microtype}
\usepackage{amsmath,amssymb,amsthm,bm}
\usepackage{booktabs,graphicx,url}
\usepackage[hidelinks]{hyperref}
\usepackage{enumitem}
\setlist{nosep,leftmargin=1.4em}
\theoremstyle{plain}
\newtheorem{proposition}{Proposition}

\theoremstyle{definition}
\newtheorem{definition}{Definition}

\newcommand{\Skstar}{11}
\newcommand{\SoneBad}{421}
\newcommand{\SoneLsMed}{1.33}
\newcommand{\SoneLsNinety}{13.8}
\newcommand{\SoneLsMax}{939}
\newcommand{\SoneUnsupZero}{0.457}
\newcommand{\SoneUnsupTen}{0.0021}
\newcommand{\SoneUnsupTruth}{0.002}
\newcommand{\SoneUnsupFive}{0.0116}

\newcommand{\SoneTruthZero}{0.024}

\newcommand{\StwoFullSix}{0.743}
\newcommand{\StwoHalfSix}{0.087}
\newcommand{\StwoFullEight}{0.371}
\newcommand{\StwoSealedSix}{0.0020}
\newcommand{\StwoMCd}{0.0058}
\newcommand{\StwoMCexact}{0.0051}
\newcommand{\SthreeHalvCovSep}{0.210}
\newcommand{\SthreeAICovSep}{0.919}
\newcommand{\SthreeHalvFWER}{0.010}
\newcommand{\SthreeResetFWER}{0.231}
\newcommand{\SthreeAIFWER}{0.338}
\newcommand{\SthreeAImFDR}{0.027}
\newcommand{\SthreeHalvCovB}{0.211}
\newcommand{\SthreeAICovB}{0.921}

\newcommand{\SfourNs}{312}
\newcommand{\SfourAggOne}{0.849}
\newcommand{\SfourAggOneBig}{1.000}
\newcommand{\SfourReqOne}{31,200}

\newcommand{\SfiveDetEight}{0.452}
\newcommand{\SfiveDelayEight}{490}
\newcommand{\SfiveDelayTwenty}{55}
\newcommand{\SfiveFailDelEight}{100.2}
\newcommand{\SfiveFailMonEight}{84.1}
\newcommand{\SfiveFailDelTwenty}{219.9}
\newcommand{\SfiveFailMonTwenty}{40.6}
\newcommand{\SfivePairedEightAtTen}{+6.9}
\newcommand{\SfivePairedTwentyAtTwenty}{+9.5}
\newcommand{\SfivePairedTwentyAtFifteen}{-37.7}

\newcommand{\TV}{\mathrm{TV}}
\newcommand{\E}{\mathbb{E}}
\newcommand{\Prob}{\mathbb{P}}

\title{\LARGE\bfseries Conformance as a Distribution:\\ \Large Delegation Contracts, Error Budgets, and the Allocation of Evaluation Authority in Use-Based AI Regulation}
\author{John deVadoss\\ \normalsize InterWork Alliance, Washington, DC\\ \normalsize johnd@ieee.org}
\date{}

\begin{document}
\maketitle
\thispagestyle{empty}

\begin{abstract}
\noindent\textbf{\textit{Abstract}---Use-based AI regulation assigns obligations by deployment risk but lacks an evidentiary object capable of carrying those obligations. A benchmark score is a property of a model; conformance of an AI deployment is a failure rate under a declared distribution of contexts, bounded by a scope and subject to expiry. This paper takes the delegation contract of prior work---a frozen configuration, a scoped claim, and an executable revision rule---and studies it under strategic interaction between a deployer, an auditor, and an enforcement point. Five questions organize the analysis: whether a declared context mixture can be made truthful; whether an audit can be made unbiased against a system that recognizes evaluation traffic; whether a certification error budget can sustain a long-lived deployment; whether an aggregate audit can support a per-subgroup guarantee; and whether an incident metric that counts system failures but not fallback costs ranks arrangements correctly. For each we give a mechanism and a proposition: a proper-scoring penalty on the declared mixture bounds the excess deployment failure of a best-responding deployer by $O(\lambda^{-1/2})$ in the penalty weight; a commit-reveal audit drawn from deployment traffic is unbiased by construction and scores the evaluator with the same sample; a halving error budget drives certification coverage below $(1-p)^n$ while $\alpha$-investing sustains it under a marginal false-certification guarantee; a per-subgroup guarantee costs total audit size scaling as $1/m$ in subgroup mass; and any failure-only incident metric prefers withdrawal for every fallback cost, including where total loss ranks the arrangements oppositely. Five preregistered synthetic studies quantify each result. The institutional translation locates enforcement at distribution, capability-access, monetization and risk-transfer choke points, treats expiry of the conformance record as a standard of care, and contrasts conformance-as-a-rate with the static conformity assessment of the EU AI Act. All numerical evidence is synthetic; no language model or production system is evaluated.}

\smallskip\noindent\textbf{\textit{Index Terms}---AI regulation, conformance, delegation, error budgets, audit, liability, subgroup guarantees, proper scoring rules, commitment schemes.}
\end{abstract}

\section{Introduction}
A regulatory regime that assigns obligations by the \emph{use} to which an AI system is put, rather than by the artifact from which it is built, needs an object that can carry an obligation. That object cannot be a benchmark score. A score is a property of a model measured under one protocol; an obligation attaches to a \emph{deployment}: a model together with its prompts, tools, memory, resource limits, fallback and the rule that decides when the system acts. Prior work has argued that such a deployment should be evaluated through a \emph{delegation contract}: a frozen configuration, a scoped claim about its failure rate under a declared mixture of contexts, and an executable rule that withdraws the claim when its assumptions no longer hold \cite{devadoss2026beyond}. That work is single-agent decision theory: an evaluator chooses experiments to minimize decision regret at stated cost.

The present paper studies the same object once a second party enters. A deployer who declares the mixture has an interest in the declaration. A system under audit may recognize audit traffic. A firm that re-enters a deployment many times exhausts any fixed error budget. An aggregate audit says nothing about a small subgroup. An incident regime that counts failures but not the cost of the human fallback rewards the wrong arrangement. None of these is a statistical problem in the ordinary sense; each is a problem of who declares what, who holds which budget, and what the enforcement point can verify. We call the resulting object a \emph{regulatory conformance claim} and give five mechanisms that make it truthful, verifiable, sustainable, fair and enforceable, each with a proposition and a preregistered synthetic study.

The thesis can be stated in one sentence. \emph{Conformance of an AI deployment is a distribution, not a property}: a failure rate under a declared mixture, holding at a declared confidence within a declared scope, until a declared expiry. The paper's contribution is to show what follows institutionally when that definition is taken seriously, and to locate the places where the definition cannot be enforced by statistics alone.

The paper is organized as follows. Section~\ref{sec:related} positions the work. Section~\ref{sec:claim} defines the conformance claim. Section~\ref{sec:mech} gives the five mechanisms and their propositions. Section~\ref{sec:adv} pairs each mechanism with the attack it answers and names the attacks it does not. Section~\ref{sec:exp} reports the studies. Section~\ref{sec:doctrine} gives the institutional and doctrinal translation. Section~\ref{sec:limits} states limitations and a falsifiable next study.

\section{Relationship to Existing Work}\label{sec:related}
\emph{Evaluation validity and agent evaluation.} Validity-centered evaluation \cite{salaudeen2025measurement} and agent evaluation that accounts for cost and reproducibility \cite{kapoor2024agents} connect measurements to claims. The delegation contract \cite{devadoss2026beyond} binds a claim to an executable policy and a revision rule, and introduces two limits we use throughout: no universal ranking under crossing losses, and a regret floor of half the untested deployment mass. We add the strategic layer: who declares, who audits, who pays.

\emph{Risk control and anytime-valid inference.} Learn-then-Test \cite{angelopoulos2021learn} connects risk control to multiple testing; time-uniform confidence sequences \cite{howard2021time}, e-values \cite{vovk2021evalues} and game-theoretic statistics \cite{ramdas2023game} supply sequentially valid monitoring. $\alpha$-investing \cite{foster2008alpha} and $p$-value weighting \cite{genovese2006false} supply budget rules across a sequence of tests. We use these as components and ask how the budget should be allocated between institutions.

\emph{Subgroup guarantees.} Multicalibration \cite{hebertjohnson2018multicalibration} and the audit of fairness gerrymandering \cite{kearns2018preventing} give guarantees over a class of subgroups rather than one. We use their sample-complexity logic to price a per-subgroup regulatory guarantee.

\emph{Strategic behavior under evaluation.} Language models have been shown to behave differently when they infer they are being evaluated or trained \cite{greenblatt2024alignment,vanderweij2024sandbagging}. Our sealed-audit construction is a response to that evidence; it uses a binding, hiding commitment \cite{naor1991bit} to make audit membership unobservable at the time of action.

\emph{Policy frameworks.} Frontier-model regulation \cite{anderljung2023frontier}, safety cases \cite{clymer2024safety}, the NIST AI RMF \cite{nistairmf2023}, the EU AI Act's conformity assessment \cite{euaiact2024} and the use-based alternative in \cite{devadoss2026spectrum} are the regimes against which Section~\ref{sec:doctrine} is written. Liability versus regulation as instruments for safety follows \cite{shavell1984liability}.

\section{The Regulatory Conformance Claim}\label{sec:claim}
Let $\mathcal{C}=\{1,\dots,C\}$ be a finite set of deployment contexts. A \emph{configuration} $a$ fixes the model version, prompts, tools, memory and resource limits; its identity is an immutable digest. Under configuration $a$, a task in context $c$ fails with probability $p_c\in[0,1]$. Failures have unit harm; severity weights are a straightforward extension and are suppressed. The \emph{deployment mixture} $q^*\in\Delta^{C}$ is the distribution of contexts the configuration actually faces; the \emph{declared mixture} $\hat q\in\Delta^C$ is what the deployer files. For any mixture $q$, write $\bar p(q)=\sum_c q_c p_c$.

\begin{definition}[Conformance claim]\label{def:claim}
A conformance claim is a tuple $\kappa=(a,\hat q,\tau,\alpha,\beta,H,\pi_f,\Pi)$: a configuration digest $a$; a declared mixture $\hat q$; a failure threshold $\tau$; a certification level $\alpha$ and a monitoring level $\beta$; a horizon $H$; a fallback $\pi_f$; and a permissible policy class $\Pi$. The claim asserts $\bar p(\hat q)\le\tau$ for configuration $a$, certified at level $\alpha$ by a fresh audit, monitored at level $\beta$ for $H$ arrivals, and withdrawn in favour of $\pi_f$ on alarm, configuration mismatch or expiry.
\end{definition}

Certification uses the one-sided exact upper bound \cite{clopper1934use}: after $k$ failures in $n$ audit trials, $U(k,n)=\mathrm{Beta}^{-1}(1-\alpha;\,k+1,\,n-k)$, and the claim is certified iff $U\le\tau$, equivalently iff $k\le k^*(n,\tau,\alpha)$. Write $P_n(q)=\Prob\{\mathrm{Bin}(n,\bar p(q))\le k^*\}$ for the certification probability when the audit draws contexts from $q$; $P_n$ is strictly decreasing in $\bar p(q)$. The exact test guarantees $P_n(q)\le\alpha$ whenever $\bar p(q)\ge\tau$.

Two facts from \cite{devadoss2026beyond} are used without proof. First, if the audit does not reach a region of deployment mass $\rho$, worst-case regret is at least $\rho/2$ regardless of precision elsewhere. Second, a claim is valid only for the frozen configuration; a mismatch invalidates automatic reuse. The first makes $\hat q$ the regulable quantity; the second makes $a$ the regulable identity.

\section{Five Mechanisms}\label{sec:mech}

\subsection{Truthful Declaration of the Mixture}\label{sec:m1}
The deployer chooses $\hat q$. The audit samples contexts from $\hat q$, so a deployer who shifts mass toward low-failure contexts lowers $\bar p(\hat q)$ and raises $P_n(\hat q)$ without changing the deployment. The enforcement point observes realized traffic $c_1,\dots,c_m\sim q^*$ and charges a penalty $\lambda\,\bar S(\hat q)$ where $\bar S(\hat q)=m^{-1}\sum_i S(\hat q,c_i)$ and $S$ is a strictly proper scoring rule \cite{gneiting2007strictly}; we use the quadratic score $S(\hat q,c)=\sum_j(\hat q_j-\mathbf 1\{c=j\})^2$ \cite{brier1950verification}. The deployer earns $V$ on certification and maximizes
\begin{equation}\label{eq:payoff}
\Pi(\hat q)=V\,P_n(\hat q)-\lambda\,\E_{c\sim q^*}\big[S(\hat q,c)\big].
\end{equation}

\begin{proposition}[Declaration bounds]\label{prop:1}
Let $p_{\max}=\max_c p_c$, $p_{\min}=\min_c p_c$ and $\TV(\hat q,q^*)=\tfrac12\|\hat q-q^*\|_1$.
\begin{enumerate}[label=(\roman*)]
\item (Transfer) $|\bar p(\hat q)-\bar p(q^*)|\le\TV(\hat q,q^*)\,(p_{\max}-p_{\min})$.
\item (Properness) $\E_{c\sim q^*}S(\hat q,c)=\|\hat q-q^*\|_2^2+1-\|q^*\|_2^2$, uniquely minimized at $\hat q=q^*$.
\item (Best response) Every maximizer $\hat q^\lambda$ of \eqref{eq:payoff} satisfies $\|\hat q^\lambda-q^*\|_2\le\sqrt{V/\lambda}$, hence $\TV(\hat q^\lambda,q^*)\le\tfrac12\sqrt{CV/\lambda}$ and
\[
\bar p(q^*)-\bar p(\hat q^\lambda)\ \le\ \tfrac12\,(p_{\max}-p_{\min})\sqrt{CV/\lambda}.
\]
\item (No exact truthfulness) If $P_n$ is differentiable at $q^*$ with $\nabla P_n(q^*)$ not orthogonal to the simplex, then for every finite $\lambda$ there is a profitable deviation from $q^*$; if $P_n$ is $L$-Lipschitz on the simplex, any maximizer in the interior of the simplex satisfies $\|\hat q^\lambda-q^*\|_2\le VL/(2\lambda)$.
\end{enumerate}
\end{proposition}
\begin{proof}
(i) Since $\sum_c(\hat q_c-q^*_c)=0$, for any constant $\mu$, $\bar p(\hat q)-\bar p(q^*)=\sum_c(\hat q_c-q^*_c)(p_c-\mu)$. Take $\mu=(p_{\max}+p_{\min})/2$; then $|p_c-\mu|\le(p_{\max}-p_{\min})/2$ and the sum is bounded by $\|\hat q-q^*\|_1(p_{\max}-p_{\min})/2$. (ii) Expand $S$ and take expectations: $\E S=\sum_j\hat q_j^2-2\sum_j\hat q_jq^*_j+1=\|\hat q-q^*\|_2^2+1-\|q^*\|_2^2$. (iii) At a maximizer, $\Pi(\hat q^\lambda)\ge\Pi(q^*)$, i.e.\ $\lambda\|\hat q^\lambda-q^*\|_2^2\le V(P_n(\hat q^\lambda)-P_n(q^*))\le V$. The $\TV$ bound follows from $\|x\|_1\le\sqrt C\|x\|_2$ on $\mathbb R^C$, and the last display from (i) together with the fact that a profitable deviation lowers $\bar p$. (iv) The gain from a deviation $\hat q=q^*+\epsilon v$ along a feasible direction $v$ is first-order in $\epsilon$ while the penalty is second-order, so for small $\epsilon$ the deviation is profitable at any finite $\lambda$. At an interior maximizer the first-order condition $2\lambda(\hat q^\lambda-q^*)=V\,\Pi\nabla P_n(\hat q^\lambda)$ holds with $\Pi$ the projection onto the tangent space of the simplex, and $\|\Pi\nabla P_n\|_2\le L$.
\end{proof}

The institutional content is in (iii) and (iv). A scoring-rule penalty cannot buy exact truthfulness at finite cost, but it bounds the \emph{excess deployment failure} that a certified claim can hide, by a quantity that shrinks as $\lambda^{-1/2}$ in general and as $\lambda^{-1}$ under smoothness. The regulator is therefore buying a bound on unsupported transfer, not a guarantee of honesty, and should set $\lambda$ by the excess failure it is willing to tolerate. The realized mixture is observable at exactly the choke points where \cite{devadoss2026spectrum} locates enforcement---cloud and platform logs, app-store telemetry---which is what makes the penalty assessable.

\subsection{Sealed Audits by Commitment}\label{sec:m2}
A fresh audit removes selection bias but not strategic bias: a system that can tell audit traffic from deployment traffic can behave better on the former \cite{greenblatt2024alignment,vanderweij2024sandbagging}. Model the \emph{distinguishable channel} as follows: audit samples are generated separately from deployment, each is recognized as audit traffic with probability $d$, and a recognized sample fails with probability $p(1-\varepsilon)$ instead of $p$. The \emph{sealed channel} instead draws the audit as a uniformly random subset $I$ of size $n$ from the $N$ deployment arrivals; the auditor publishes a binding, hiding commitment $\mathrm{com}=\mathrm{Commit}(I,r)$ before the first arrival \cite{naor1991bit}, outcomes for $i\in I$ are adjudicated after the fact, and $(I,r)$ is revealed with the audit record.

\begin{proposition}[Sealed audit]\label{prop:2}
Assume the system's action at arrival $i$ depends only on its own observations up to $i$ and on public information, and that the commitment is hiding and binding.
\begin{enumerate}[label=(\roman*)]
\item On the sealed channel, $\mathbf 1\{i\in I\}$ is independent of the system's view at $i$; hence each audited outcome has failure probability equal to the deployment failure probability, the audit estimator $K/n$ is unbiased for $\bar p(q^*)$, and the certification error is at most $\alpha$ under the null $\bar p(q^*)\ge\tau$.
\item On the distinguishable channel, the false-certification probability equals $\Prob\{\mathrm{Bin}(n,p(1-d\varepsilon))\le k^*\}$, which is not controlled by $\alpha$ and tends to one as $d\varepsilon\to1$ whenever $p(1-d\varepsilon)<\tau$ is reachable.
\item (Evaluator scoring) An evaluator who certified the claim at rate $\hat r$ is scored by the quadratic score against the sealed outcomes; its expectation $p(1-p)+(\hat r-p)^2$ is uniquely minimized at $\hat r=p$, the deployment rate. Scoring against the distinguishable channel instead rewards $\hat r=p(1-d\varepsilon)$.
\end{enumerate}
\end{proposition}
\begin{proof}
(i) $I$ is drawn uniformly and independently of all arrivals; the commitment is hiding, so the public transcript before the reveal is computationally independent of $I$; the system's view is therefore independent of membership, and conditioning on $i\in I$ does not change its failure probability. Binding prevents the auditor from altering $I$ after outcomes are seen, which is the selection-validity condition of \cite{devadoss2026beyond}. The error bound is the Clopper--Pearson guarantee at the deployment rate. (ii) Each distinguishable sample fails with marginal probability $d\,p(1-\varepsilon)+(1-d)p=p(1-d\varepsilon)$ independently; the certification event is $\{K\le k^*\}$. Monotonicity of the binomial c.d.f.\ in $p$ gives the limit. (iii) Expand $\E(\hat r-Y)^2$ with $Y\sim\mathrm{Bernoulli}(p)$.
\end{proof}

The construction requires that outcomes of sampled deployment traffic can be adjudicated after the fact and that adjudication activity is not itself observable to the system at decision time. Both are costs, not obstacles; they are the price of an audit that cannot be gamed by recognition, and they enter the lifecycle loss. Part (iii) is the result we regard as most useful institutionally: the sealed sample that audits the deployer also scores the evaluator, so accreditation of evaluators by decision consequence \cite{devadoss2026beyond} needs no second data stream.

\subsection{Sustaining the Error Budget Across Re-entry}\label{sec:m3}
A deployment that is withdrawn and re-enters $J$ times is a family of $J$ certification tests. Three schedules are in use or proposed. \emph{Halving} sets $\alpha_j=\alpha^*2^{-j}$, which controls the family-wise probability of any false certification by the union bound. \emph{Fixed reset} sets $\alpha_j=\alpha^*$ and controls nothing across episodes. \emph{$\alpha$-investing} \cite{foster2008alpha} holds a wealth $W_j$, spends $\alpha_j\le W_{j-1}/(1+W_{j-1})$, earns $\omega\le\alpha^*$ on each certification and loses $\alpha_j/(1-\alpha_j)$ otherwise; it controls the marginal false-discovery rate $\mathrm{mFDR}_\eta=\E[\#\text{false certifications}]/(\E[\#\text{certifications}]+\eta)$ at $\alpha^*$ with $\eta=1-\alpha^*$.

\begin{proposition}[Budget schedules]\label{prop:3}
Let an acceptable policy fail at rate $p_a<\tau$ and let coverage at episode $j$ be $\mathrm{cov}_j=P_n$ evaluated at $\bar p=p_a$ with level $\alpha_j$.
\begin{enumerate}[label=(\roman*)]
\item Under halving, $\Prob\{\text{any false certification in }J\text{ episodes}\}\le\alpha^*$ for all $J$ without independence, and $\mathrm{cov}_j$ is nonincreasing with $\lim_{j\to\infty}\mathrm{cov}_j\le(1-p_a)^n$.
\item Under fixed reset with $u$ unacceptable boundary episodes ($\bar p=\tau$), the family-wise false-certification probability is $1-(1-\alpha_u)^u$ where $\alpha_u=P_n(\tau)\le\alpha^*$, which tends to one as $u$ grows.
\item Under $\alpha$-investing with $\omega\le\alpha^*$, $\mathrm{mFDR}_{1-\alpha^*}\le\alpha^*$ at every stopping time \cite{foster2008alpha}; the family-wise probability is not controlled.
\item (Firm-level allocation) If a regulator fixes a firm-wide $\alpha^*$ and the firm allocates $\alpha_i=w_i\alpha^*$ across deployments with $\sum_iw_i\le1$, the family-wise guarantee is preserved for any weights \cite{genovese2006false}; the firm may therefore concentrate the budget where its private information says certification is most valuable.
\end{enumerate}
\end{proposition}
\begin{proof}
(i) Union bound over $\sum_j\alpha^*2^{-j}\le\alpha^*$. As $\alpha_j\downarrow0$, $k^*(n,\tau,\alpha_j)$ is nonincreasing and reaches $0$ once $U(0,n)=1-\alpha_j^{1/n}>\tau$; at $k^*=0$, $\mathrm{cov}_j=(1-p_a)^n$, and for smaller $\alpha_j$ no certification is possible. (ii) Independence across episodes of the audit draws. (iii) Theorem of \cite{foster2008alpha}. (iv) Weighted Bonferroni.
\end{proof}

The policy reading is a trade, not a ranking. Halving gives the strongest guarantee and starves the deployment; $\alpha$-investing sustains coverage under a guarantee on the \emph{fraction} of certified deployments that are unacceptable, which is arguably the quantity a regulator cares about; fixed reset, the de facto practice of re-running a level-$\alpha$ test at each re-entry, gives nothing. Part (iv) says the regulator can set the family-wide level and leave allocation to the firm without loss of guarantee, which is the decentralization a use-based regime needs. E-processes \cite{vovk2021evalues,ramdas2023game} give the within-episode anytime-valid audit and combine multiplicatively; the cross-episode allocation remains a budgeting choice and is what Proposition~\ref{prop:3} addresses.

\subsection{Pricing a Per-Subgroup Guarantee}\label{sec:m4}
Let a subgroup $G$ of deployment mass $m$ fail at rate $p_G$ while the remainder fails at $p_0<\tau$. The aggregate rate is $\bar p=(1-m)p_0+mp_G$.

\begin{proposition}[Subgroup cost]\label{prop:4}
\begin{enumerate}[label=(\roman*)]
\item An aggregate audit can certify while $p_G>\tau$ whenever $m<(\tau-p_0)/(p_G-p_0)$, with probability $P_n$ at $\bar p$, decreasing in $m$.
\item Let $n_s^*$ be the least sample size at which an acceptable subgroup ($p_G=p_a$) is certified with probability at least $1-\beta_s$ at level $\alpha$. Under proportional sampling the expected subgroup sample in an audit of size $n$ is $nm$, so the total audit required for the subgroup guarantee is at least $n_s^*/m$.
\item For a class $\mathcal G$ of subgroups with minimum mass $m_{\min}$, a stratified audit with a union bound over $\mathcal G$ requires total size $\Omega\big(n_s^*(\alpha/|\mathcal G|)/m_{\min}\big)$; guarantees over computationally identifiable classes follow the sample-complexity logic of \cite{hebertjohnson2018multicalibration,kearns2018preventing}.
\end{enumerate}
\end{proposition}
\begin{proof}
(i) Solve $\bar p<\tau$ for $m$; monotonicity of $P_n$. (ii) The subgroup count in a proportional audit is $\mathrm{Bin}(n,m)$ with mean $nm$; a certification at power $1-\beta_s$ needs $n_s^*$ subgroup samples. (iii) Bonferroni over $|\mathcal G|$ claims at level $\alpha/|\mathcal G|$ raises $n_s^*$; dividing by $m_{\min}$ gives the total.
\end{proof}

The regulatory consequence is a price list. A per-subgroup guarantee for a group of mass one percent costs one hundred times the subgroup's own audit, and an aggregate audit of ordinary size is not merely weak evidence for the subgroup claim but no evidence at all (Section~\ref{sec:s4}). A tier that covers decisions affecting individuals must either budget for stratified audits or state plainly that its guarantee is aggregate.

\subsection{Counting the Fallback}\label{sec:m5}
Let the system fail at rate $p_1$ after a deterioration, let the fallback fail at $p_f<p_1$ and cost $f$ per task, and let $T$ tasks arrive after the point at which an arrangement may withdraw. Withdrawal changes expected failures by $-(p_1-p_f)T$ and adds $fT$ to loss.

\begin{proposition}[Failure-only metrics]\label{prop:5}
Let an incident metric $M$ be any function of the arrangement that is nondecreasing in expected failures and independent of fallback volume. Then $M$ ranks withdrawal weakly above continued delegation for every $f\ge0$, whereas total loss ranks delegation above withdrawal iff $f>(p_1-p_f)$. Hence for all $f>p_1-p_f$ the two rankings disagree.
\end{proposition}
\begin{proof}
Expected failures under withdrawal are lower by $(p_1-p_f)T>0$, so $M$ prefers withdrawal regardless of $f$. Total loss differs by $(f-(p_1-p_f))T$.
\end{proof}

The result is elementary and the policy error it describes is common: serious-incident reporting that counts system failures and not the volume or cost of human fallback induces withdrawal exactly where withdrawal is the costlier outcome. The crossing $f^*=p_1-p_f$ is the number an enforcement regime must know before it can interpret a falling incident count as improvement.

\section{Adversarial Analysis}\label{sec:adv}
A mechanism is only as strong as the attack it answers. Table~\ref{tab:adv} pairs each mechanism with its attack and states which attacks remain open.

\begin{table}[t]
\caption{Attacks and answers}
\label{tab:adv}
\centering\footnotesize
\begin{tabular}{@{}p{2.1cm}p{2.5cm}p{2.5cm}@{}}
\toprule
Attack & Answer & Residual \\
\midrule
Declare a narrow or favourable mixture & Scoring-rule penalty on realized traffic (Prop.~\ref{prop:1}) & Excess failure bounded, not zero; requires traffic observable at the choke point \\
Recognize audit traffic and behave better & Commit-reveal audit from deployment traffic (Prop.~\ref{prop:2}) & Requires ex post adjudication; adjudication must be unobservable at decision time \\
Split one deployment into many to reset budgets & Firm-level budget with weighted allocation (Prop.~\ref{prop:3}(iv)) & Requires the enforcement point to know the firm, not only the deployment \\
Gerrymander subgroups & Class-based guarantees with stratified audits (Prop.~\ref{prop:4}(iii)) & Cost scales with $1/m_{\min}$ and $\log|\mathcal G|$; the class must be fixed ex ante \\
Report incident counts only & Loss metric that includes fallback volume (Prop.~\ref{prop:5}) & Requires fallback cost $f$ to be disclosed or estimated \\
Shade the configuration & Immutable digest in the record; runtime guard operated by the platform \cite{devadoss2026beyond} & Guard detects declared mismatches only, not unknown assumptions \\
\bottomrule
\end{tabular}
\end{table}

Two attacks have no answer in this paper. A deployer who controls the \emph{adjudication} of sealed outcomes can corrupt Proposition~\ref{prop:2} at the source; adjudication must therefore be independent, which is an institutional requirement and not a statistical one. And a system whose deployment distribution depends on its own policy---performative in the sense of \cite{perdomo2020performative}---changes $q^*$ between audit and deployment; the scoring-rule penalty then charges the deployer for a drift it did not declare but may have caused. We flag both as open.

\section{Experiments}\label{sec:exp}
\subsection{Scope, Preregistration and Reproducibility}
All five studies are synthetic. No language model, production workflow, human participant or model API is evaluated; the studies establish the behaviour of the mechanisms under declared simulators. A protocol file fixing every parameter, the base seed $20261019$, and the primary outputs was written before the first run. Two amendments were made after the first run and are recorded in the protocol with their reasons: Study~1 gained the mechanism-outcome curve needed to state Proposition~\ref{prop:1}(iii) as a bound on realized harm, and Study~3 gained a boundary scenario ($p=\tau$) because the preregistered separated scenario ($p=0.08$) cannot exhibit family-wise inflation with an exact test at $n=400$. Intervals are pointwise 95\% Student-$t$ over independent worlds or replicates, with no multiplicity adjustment. Where a quantity is a binomial probability it is computed exactly and the Monte Carlo check is reported. The source package regenerates every number and figure from the protocol, and an independent verification script recomputes the reported statistics from the saved results.

Throughout, $\tau=0.05$, $\alpha=0.025$ and $n=400$ unless stated; at these values $k^*=\Skstar$.

\subsection{Study 1: Declaration Under a Scoring-Rule Penalty}\label{sec:s1}
We generate $500$ worlds with $C=8$ contexts, $q^*\sim\mathrm{Dirichlet}(1,\dots,1)$ and $p_c\sim\mathrm{Uniform}(0.01,0.12)$. The deployer's declaration is $\hat q(w)=(1-w)q^*+w\,q_{\mathrm{fav}}$ with $q_{\mathrm{fav}}\propto q^*_c\mathbf 1\{p_c\le\tau\}$ and $w$ on a grid from $0$ to $1$; $V=1$. Certification probability and expected score are computed exactly.

\begin{figure*}[t]
\centering\includegraphics[width=\textwidth]{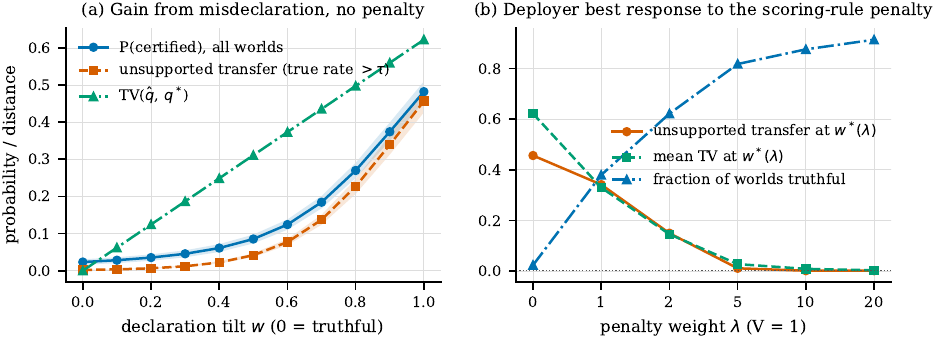}
\caption{Study 1. (a) Without a penalty, tilting the declared mixture raises the certification probability from $0.023$ to $0.48$; in the $\SoneBad$ of $500$ worlds whose true deployment rate exceeds $\tau$, the probability of an unsupported certification rises from $\SoneUnsupTruth$ to $\SoneUnsupZero$. Shaded bands are pointwise 95\% intervals over worlds. (b) Deployer best response to the penalty weight $\lambda$: unsupported transfer, mean $\TV$ at the best response, and the fraction of worlds in which the truthful declaration is optimal.}
\label{fig:1}
\end{figure*}

\begin{table}[t]
\caption{Study 1: deployer best response by penalty weight ($500$ worlds, $V=1$)}
\label{tab:s1}
\centering\footnotesize
\begin{tabular}{@{}rcccc@{}}
\toprule
$\lambda$ & Frac.\ truthful & Mean $\TV$ & P(cert.) & Unsupported \\
\midrule
0 & 0.024 & 0.622 & 0.482 & 0.4568 \\
1 & 0.380 & 0.330 & 0.384 & 0.3419 \\
2 & 0.622 & 0.146 & 0.216 & 0.1500 \\
5 & 0.818 & 0.028 & 0.071 & 0.0116 \\
10 & 0.876 & 0.009 & 0.039 & 0.0021 \\
20 & 0.914 & 0.004 & 0.031 & 0.0020 \\
\bottomrule
\end{tabular}

\end{table}

Fig.~\ref{fig:1} and Table~\ref{tab:s1} give the result. With no penalty the deployer tilts almost fully ($\SoneTruthZero$ of worlds truthful) and the unsupported-transfer rate is $\SoneUnsupZero$. At $\lambda=5$ the rate is $\SoneUnsupFive$ and at $\lambda=10$ it is $\SoneUnsupTen$, against $\SoneUnsupTruth$ under truthful declaration; mean $\TV$ at the best response falls roughly as $1/\lambda$ between $\lambda=10$ and $\lambda=20$, consistent with Proposition~\ref{prop:1}(iv). The minimal penalty that makes the truthful declaration optimal on the grid has median $\SoneLsMed$ across worlds but a heavy tail (90th percentile $\SoneLsNinety$, maximum $\SoneLsMax$), which is the finite-sample face of Proposition~\ref{prop:1}(iv): worlds exist in which a small tilt crosses the certification threshold at second-order cost. The regulator does not need those worlds to be truthful; it needs the excess failure they can hide to be small, and at $\lambda=10$ it is.

\subsection{Study 2: Sealed Versus Distinguishable Audit}\label{sec:s2}
For true deployment failure $p\in\{0.06,0.08,0.10\}$, all above $\tau$, and detection probability $d\in\{0,0.25,0.5,0.75,1\}$ with $\varepsilon=0.6$, we compute the false-certification probability on each channel exactly (Table~\ref{tab:s2}, Fig.~\ref{fig:2}). A Monte Carlo check with $20{,}000$ replicates at $p=0.08$, $d=0.5$ gives $\StwoMCd$ against the exact $\StwoMCexact$.

\begin{table}[t]
\caption{Study 2: false-certification probability, $n=400$, $\varepsilon=0.6$}
\label{tab:s2}
\centering\footnotesize
\begin{tabular}{@{}lccccc|c@{}}
\toprule
& \multicolumn{5}{c|}{Distinguishable, $d=$} & Sealed \\
$p$ & 0 & 0.25 & 0.5 & 0.75 & 1 & any $d$ \\
\midrule
0.06 & 0.002 & 0.015 & 0.087 & 0.330 & 0.743 & 0.0020 \\
0.08 & 0.000 & 0.000 & 0.005 & 0.061 & 0.371 & 0.0000 \\
0.1 & 0.000 & 0.000 & 0.000 & 0.006 & 0.122 & 0.0000 \\
\bottomrule
\end{tabular}

\end{table}

\begin{figure}[t]
\centering\includegraphics[width=\columnwidth]{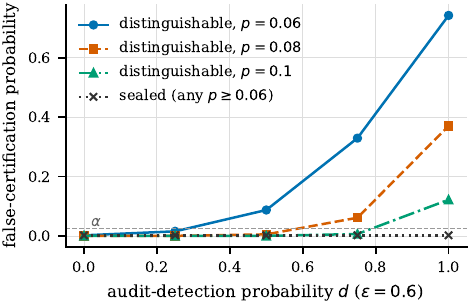}
\caption{Study 2. False certification on the distinguishable channel rises with the detection probability $d$; the sealed channel is flat at the exact-test level for every $p\ge0.06$.}
\label{fig:2}
\end{figure}

At $p=0.06$ a fully recognizing system is falsely certified with probability $\StwoFullSix$ and a half-recognizing one with $\StwoHalfSix$; the sealed channel is at $\StwoSealedSix$. At $p=0.08$ full recognition still yields $\StwoFullEight$. The effect is not a tail phenomenon; it is the main effect once $p(1-d\varepsilon)$ crosses $\tau$. The evaluator-scoring expectations of Proposition~\ref{prop:2}(iii) are reported in the source package.

\subsection{Study 3: Error Budgets Across Re-entry}\label{sec:s3}
Over $J=24$ re-entry episodes and $5{,}000$ replicate sequences, each episode's policy is acceptable ($p=0.02$) with probability $0.7$ and otherwise unacceptable at $p=0.08$ (separated) or $p=\tau$ (boundary); $\alpha^*=0.05$, $n=400$, $\phi=0.5$, $\omega=\alpha^*$. Failure counts are shared across schedules within a replicate.

\begin{table}[t]
\caption{Study 3: schedules over 24 re-entry episodes ($5{,}000$ replicates)}
\label{tab:s3}
\centering\footnotesize
\begin{tabular}{@{}llcccc@{}}
\toprule
Scenario & Schedule & FWER & mFDR & Coverage & Mean cert. \\
\midrule
Separated & Halving & 0.0000 & 0.0000 & 0.210 & 3.52 \\
Separated & Fixed reset & 0.0002 & 0.0000 & 0.939 & 15.78 \\
Separated & $\alpha$-investing & 0.0008 & 0.0001 & 0.919 & 15.45 \\
Boundary & Halving & 0.0096 & 0.0027 & 0.211 & 3.56 \\
Boundary & Fixed reset & 0.2306 & 0.0160 & 0.938 & 16.03 \\
Boundary & $\alpha$-investing & 0.3378 & 0.0270 & 0.921 & 15.91 \\
\bottomrule
\end{tabular}

\end{table}

\begin{figure*}[t]
\centering\includegraphics[width=\textwidth]{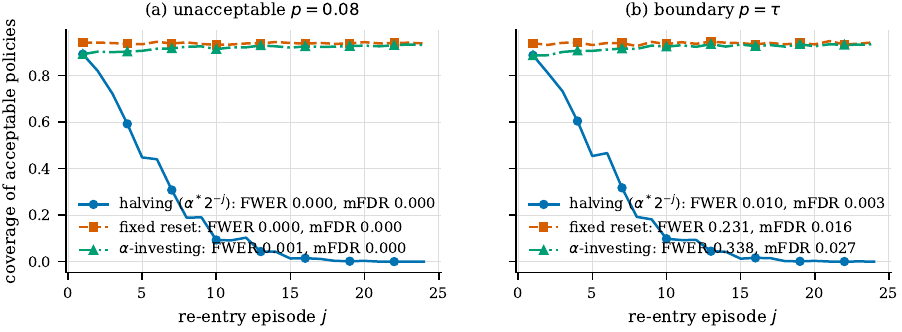}
\caption{Study 3. Coverage of acceptable policies by episode under the three schedules, with the realized family-wise false-certification probability (FWER) and marginal false-discovery rate (mFDR). (a) Unacceptable policies at $p=0.08$. (b) Unacceptable policies at the boundary $p=\tau$.}
\label{fig:3}
\end{figure*}

Table~\ref{tab:s3} and Fig.~\ref{fig:3} show the trade. Halving holds the family-wise rate at $\SthreeHalvFWER$ in the boundary scenario but covers only $\SthreeHalvCovB$ of acceptable episodes, with coverage near zero after the fifteenth re-entry, as Proposition~\ref{prop:3}(i) predicts at $n=400$ where $(1-0.02)^{400}\approx3\times10^{-4}$. Fixed reset covers $0.94$ but falsely certifies at least once in $\SthreeResetFWER$ of sequences. $\alpha$-investing covers $\SthreeAICovB$ with mFDR $\SthreeAImFDR\le\alpha^*$ and a family-wise rate of $\SthreeAIFWER$, which it does not claim to control. In the separated scenario none of the schedules falsely certifies at a detectable rate and the only difference is coverage ($\SthreeHalvCovSep$ for halving against $\SthreeAICovSep$ for $\alpha$-investing); the boundary scenario is the one on which the guarantees are stated and the one a regulator should examine.

\subsection{Study 4: Aggregate Audits and Subgroups}\label{sec:s4}
A subgroup of mass $m\in\{0.01,\dots,0.2\}$ fails at $0.15$ while the remainder fails at $0.02$. Table~\ref{tab:s4} and Fig.~\ref{fig:4} give the exact aggregate certification probability at $n\in\{400,1600\}$, the least subgroup sample $n_s^*$ at which an acceptable subgroup is certified with probability $0.8$, and the resulting total audit under proportional sampling.

\begin{table}[t]
\caption{Study 4: aggregate certification while the subgroup fails at $0.15$; required total audit for a subgroup guarantee ($n_s^*=\SfourNs$)}
\label{tab:s4}
\centering\footnotesize
\begin{tabular}{@{}ccccr@{}}
\toprule
$m$ & $\bar p$ & P(cert.), $n{=}400$ & P(cert.), $n{=}1600$ & $n_s^*/m$ \\
\midrule
0.01 & 0.0213 & 0.849 & 1.000 & 31,200 \\
0.02 & 0.0226 & 0.801 & 1.000 & 15,600 \\
0.05 & 0.0265 & 0.628 & 0.998 & 6,240 \\
0.10 & 0.0330 & 0.330 & 0.910 & 3,120 \\
0.20 & 0.0460 & 0.042 & 0.090 & 1,560 \\
\bottomrule
\end{tabular}

\end{table}

\begin{figure*}[t]
\centering\includegraphics[width=\textwidth]{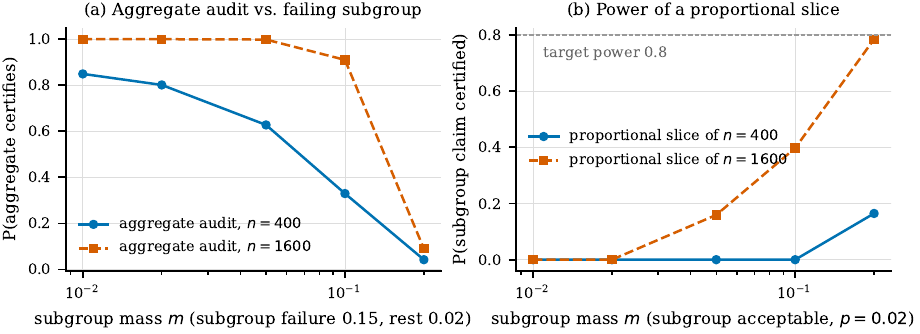}
\caption{Study 4. (a) Probability that an aggregate audit certifies a deployment whose subgroup of mass $m$ fails at three times the threshold. (b) Probability that the subgroup slice of a proportional audit certifies an \emph{acceptable} subgroup: at $n=400$ the slice has no power for any $m\le0.1$.}
\label{fig:4}
\end{figure*}

At $m=0.01$ an aggregate audit of $400$ certifies with probability $\SfourAggOne$ and one of $1600$ with probability $\SfourAggOneBig$, while the subgroup fails at three times the threshold; the per-subgroup guarantee would require a total audit of $\SfourReqOne$. Fig.~\ref{fig:4}(b) makes the stronger point: the subgroup slice of a $400$-sample audit cannot certify even an acceptable subgroup for any $m\le0.1$, so the aggregate audit is not weak evidence about the subgroup but none.

\subsection{Study 5: Incident Metrics and Fallback Cost}\label{sec:s5}
Over $2{,}000$ streams of $2{,}000$ arrivals, the failure rate shifts at arrival $1{,}000$ from $0.02$ to $p_1\in\{0.08,0.20\}$; the fallback fails at $0.01$ and costs $f$ per task. The monitored arrangement runs the mixture supermartingale of \cite{devadoss2026beyond} at $\beta=0.05$ and withdraws to fallback on alarm. Under the large shift the monitor alarms in every stream with mean delay $\SfiveDelayTwenty$ arrivals; under the small shift it alarms in $\SfiveDetEight$ of streams with mean delay $\SfiveDelayEight$.

\begin{table}[t]
\caption{Study 5: failures and total loss, $2{,}000$ streams}
\label{tab:s5}
\centering\footnotesize
\begin{tabular}{@{}ccccccr@{}}
\toprule
$p_1$ & $f$ & Fail.\ del. & Fail.\ mon. & Loss del. & Loss mon. & $\Delta$ \\
\midrule
0.08 & 0.00 & 100.2 & 84.1 & 100.2 & 84.1 & -16.1 \\
0.08 & 0.06 & 100.2 & 84.1 & 100.2 & 97.9 & -2.3 \\
0.08 & 0.08 & 100.2 & 84.1 & 100.2 & 102.5 & +2.3 \\
0.08 & 0.15 & 100.2 & 84.1 & 100.2 & 118.6 & +18.4 \\
0.08 & 0.20 & 100.2 & 84.1 & 100.2 & 130.1 & +30.0 \\
0.08 & 0.30 & 100.2 & 84.1 & 100.2 & 153.2 & +53.0 \\
0.2 & 0.00 & 219.9 & 40.6 & 219.9 & 40.6 & -179.3 \\
0.2 & 0.06 & 219.9 & 40.6 & 219.9 & 97.2 & -122.7 \\
0.2 & 0.08 & 219.9 & 40.6 & 219.9 & 116.1 & -103.8 \\
0.2 & 0.15 & 219.9 & 40.6 & 219.9 & 182.2 & -37.7 \\
0.2 & 0.20 & 219.9 & 40.6 & 219.9 & 229.4 & +9.5 \\
0.2 & 0.30 & 219.9 & 40.6 & 219.9 & 323.9 & +103.9 \\
\bottomrule
\end{tabular}

\end{table}

\begin{figure*}[t]
\centering\includegraphics[width=\textwidth]{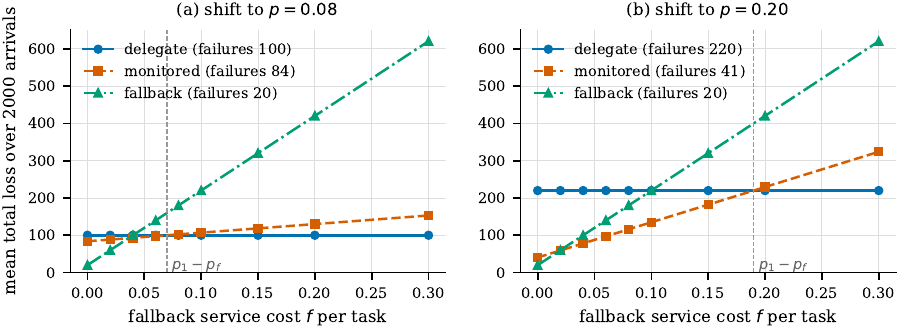}
\caption{Study 5. Mean total loss by fallback service cost $f$ for delegation throughout, the monitored arrangement, and fallback throughout. Failure counts, which do not vary with $f$, are in the legend. The dashed line is the analytic crossing $p_1-p_f$ of Proposition~\ref{prop:5}.}
\label{fig:5}
\end{figure*}

Table~\ref{tab:s5} and Fig.~\ref{fig:5} confirm Proposition~\ref{prop:5} in both directions. Under the large shift the monitored arrangement cuts failures from $\SfiveFailDelTwenty$ to $\SfiveFailMonTwenty$ at every $f$, yet its paired total loss exceeds delegation by $\SfivePairedTwentyAtTwenty$ at $f=0.20$ after being lower by $\SfivePairedTwentyAtFifteen$ at $f=0.15$; the crossing sits at the analytic $0.19$. Under the small shift failures fall from $\SfiveFailDelEight$ to $\SfiveFailMonEight$ while the loss ranking reverses at $f\approx0.07$, and by $f=0.10$ the monitored arrangement is worse by $\SfivePairedEightAtTen$. A regime observing only the first two columns would report improvement in every row.

\section{Institutional and Doctrinal Translation}\label{sec:doctrine}
\subsection{Who Declares, Who Holds, Who Verifies}
The five mechanisms assign the parameters of Definition~\ref{def:claim} to parties. The deployer declares $\hat q$ and $a$ and is charged on the realized mixture (Section~\ref{sec:m1}). An independent adjudicator draws and seals the audit, and the same sample scores the certifying evaluator (Section~\ref{sec:m2}). The regulator fixes a firm-wide $\alpha^*$ and the firm allocates it (Section~\ref{sec:m3}). The regulator fixes the subgroup class $\mathcal G$ for each tier and the deployer pays the stratified audit (Section~\ref{sec:m4}). The enforcement point requires disclosure of $f$ and counts fallback volume in the incident record (Section~\ref{sec:m5}).

Enforcement is located where \cite{devadoss2026spectrum} locates it: at distribution (app stores and enterprise marketplaces), capability access (cloud and AI platforms), monetization (payment systems and advertising networks) and risk transfer (insurers and counterparties). Each of these can demand the conformance record as a condition of service and each can observe the realized mixture the record is scored against. A platform operating the runtime guard, rather than the deployer, closes the configuration-shading attack of Section~\ref{sec:adv}. The choke-point owners thereby become de facto verifiers; the governance of \emph{those} parties is outside this paper and we flag it as the principal institutional residual.

\subsection{Expiry as a Standard of Care}
Negligence per se attaches when a defendant violates a statute or regulation designed to protect against the type of harm that occurred \cite{restatement3torts}. A conformance record that is required by regulation and that declares a horizon $H$, a monitoring level $\beta$ and a fallback $\pi_f$ supplies such a standard. Operation past $H$ without re-entry, or past an alarm without withdrawal, is then a violation of a filed duty rather than a question for a reasonableness inquiry, and the record itself fixes the standard of care at the moment of deployment. Instruments whose validity is conditioned on continued compliance rather than fixed by time already exist: a standard airworthiness certificate remains effective only while maintenance is performed in accordance with the applicable rules \cite{cfr21181}, and accelerated drug approval is conditioned on postmarketing confirmatory studies and may be withdrawn when they fail \cite{cfr314510}. The conformance record is the AI analogue: delegated authority that lapses by its own terms. On the liability-versus-regulation question \cite{shavell1984liability}, the record's role is to make the regulatory standard sharp enough that liability can be assessed against it, which is the combination that instrument theory recommends where harms are observable but ex ante inspection is costly.

\subsection{Conformance as a Rate Versus Conformity Assessment}
The EU AI Act's conformity assessment for high-risk systems \cite{euaiact2024} is conformance as a property: assessed before placing on the market, affirmed by a CE marking, and renewed on substantial modification. Its post-market monitoring and serious-incident reporting duties attach obligations after deployment but carry no error-budget semantics and no expiry: nothing in the instrument says at what confidence a claim holds, for how long, or what fraction of certified systems may be unacceptable. The present framework supplies those semantics without displacing the Act's structure. The configuration digest operationalizes ``substantial modification'': any change to $a$ is one. The monitoring level $\beta$ and horizon $H$ give post-market monitoring a null and an expiry. The fallback-inclusive loss metric gives serious-incident reporting a denominator. We regard the difference as the difference between a certificate and a claim with a confidence and a shelf life, and we note that the Act's own text on modification is closer to the second than its practice.

\section{Limitations and a Falsifiable Next Study}\label{sec:limits}
The results are propositions and synthetic illustrations, not empirical findings about deployed systems. Proposition~\ref{prop:1} assumes the enforcement point observes realized traffic; where it does not, the penalty cannot be assessed. Proposition~\ref{prop:2} assumes independent adjudication and unobservable adjudication activity; both are institutional conditions. Proposition~\ref{prop:3} controls one of two error rates per schedule and the choice between them is a policy judgment the paper does not make. Proposition~\ref{prop:4} prices guarantees for a fixed class; a class chosen after the fact is gerrymandering. Proposition~\ref{prop:5} takes $f$ as known. Severity weights, grader error, feedback delay and performative drift are absent from the studies. The simulations were designed after the propositions were drafted and are not an independent test of them; they quantify magnitudes under declared parameters.

The next study should implement the sealed audit against a real software-agent configuration in an environment where adjudication is by executable checks, with the commitment published before each episode; vary $d$ by giving the system access to features that leak audit membership and measure the gap between the two channels; and run the $\alpha$-investing and halving schedules over at least twenty real re-entries with a held-out world family built by an independent team. The integration hypothesis is rejected if the sealed and distinguishable channels do not separate when membership is leaked, or if $\alpha$-investing's realized mFDR exceeds $\alpha^*$ over the re-entry sequence.

\section{Conclusion}
A regulatory conformance claim is a failure rate under a declared mixture, at a declared confidence, within a declared scope, until a declared expiry. We have shown what each clause of that definition costs when a second party has an interest in it: truthfulness of the mixture is bought as a bound, not a guarantee; unbiasedness of the audit is bought by commitment and adjudication; sustainability of the budget is bought by choosing which error rate to control; fairness is bought at a price inversely proportional to the mass of the group; and enforceability is bought by counting the fallback. Each is a design choice a regulator can make explicitly, and none can be made by the deployer alone.

\appendix
\section{Reproduction and Verification}
The source package contains \texttt{protocol.json} (parameters, preregistered before the first run, with two recorded amendments), \texttt{experiments.py} (the five studies; base seed $20261019$ with study offsets $100$, $200$, $300$, $400$ and $500$), \texttt{figures.py} (every figure, table and in-text number is generated from \texttt{results/}), and \texttt{verify.py}, which recomputes the binomial certification probabilities, the Study~3 error-rate statistics and the Study~5 crossings independently of the experiment module and checks them against the saved results. The implementation uses Python~3.12, NumPy~2.5 and SciPy~1.18. Two passes of \texttt{pdflatex} with \texttt{bibtex} compile the manuscript.

\bibliographystyle{ieeetr}
\bibliography{refs}
\end{document}